\documentclass[letterpaper, left=48pt, right=48pt, bottom=47pt, top=60pt, conference]{ieeeconf}  

\IEEEoverridecommandlockouts                              

\usepackage{times} 

\usepackage{amsmath} 
\usepackage{amssymb}  
\usepackage{amsfonts}
\usepackage{amsthm}
\usepackage{bm}
\usepackage[final]{changes}
\usepackage{graphicx}
\usepackage{epstopdf}
\usepackage{subfig}
\usepackage{booktabs}
\usepackage[flushleft]{threeparttable}
\usepackage{multirow}
\usepackage{makecell}
\usepackage[numbers,sort&compress]{natbib}
\usepackage[linesnumbered,ruled,vlined]{algorithm2e}
\usepackage{amssymb}
\usepackage{pifont}
\usepackage{caption}
\usepackage{balance}
\usepackage[export, demo]{adjustbox}    
\usepackage{cellspace,                  
            tabularx}

\usepackage[colorlinks,
            linkcolor=blue,
            anchorcolor=blue,
            citecolor=blue]{hyperref}

\newtheorem{thm}{Theorem}
\newtheorem{lem}{Lemma}

\theoremstyle{definition}
\newtheorem{defn}{Definition}

\title{\LARGE \bf
Collision-free Control Barrier Functions for General Ellipsoids via Separating Hyperplane
}

\author{
      Zeming~Wu,
      and Lu~Liu,~\IEEEmembership{Senior Member,~IEEE,}
      \thanks{This work was supported in part by the National Natural Science Foundation of China under Grants 62373314 and 62222318, and in part by the Research Grants Council of the Hong Kong Special Administrative Region of China under Projects CityU/11210424.}
      \thanks{The authors are with the Department of Mechanical Engineering, City University of Hong Kong, Hong Kong, SAR, China (e-mail: zemingwu5-c@my.cityu.edu.hk; luliu45@cityu.edu.hk).}
}

\begin{document}

\bstctlcite{IEEEexample:BSTcontrol}

\maketitle
\thispagestyle{empty}
\pagestyle{empty}

\begin{abstract}
  This paper presents a novel collision avoidance method for general ellipsoids based on control barrier functions (CBFs) and separating hyperplanes.
  First, collision-free conditions for general ellipsoids are analytically derived using the concept of dual cones.
  These conditions are incorporated into the CBF framework by extending the system dynamics of controlled objects with separating hyperplanes, enabling efficient and reliable collision avoidance.
  The validity of the proposed collision-free CBFs is rigorously proven, ensuring their effectiveness in enforcing safety constraints.
  The proposed method requires only single-level optimization, significantly reducing computational time compared to state-of-the-art methods.
  Numerical simulations and real-world experiments demonstrate the effectiveness and practicality of the proposed algorithm.
\end{abstract}

\section{Introduction}
\label{sec:introduction}

The rapid advancement of artificial intelligence has driven the deployment of autonomous systems into increasingly complex environments, such as self-driving vehicles in urban road networks \cite{ferranti2022distributed, huang2023decentralized,firoozi2024distributed,dai2024sailing} and robotic manipulators in cluttered production lines \cite{wen2022path,singletary2022safety,raghunathan2024pyrobocop,wei2024diffocclusion}.
In these applications, the geometric shapes of controlled objects cannot be neglected, as simplified geometric models often yield overly conservative control policies that compromise task efficiency.

Various approaches have been developed to achieve geometric-aware collision avoidance, including trajectory optimization (TO)-based methods \cite{wang2023linear,tracy2023differentiable}, model predictive control (MPC)-based methods \cite{zhang2020optimization,wu2024gpu}, and CBF-based methods \cite{singletary2022safety,thirugnanam2022duality,dai2023safe}.
TO-based and MPC-based methods achieve geometric-aware collision avoidance by incorporating collision-free constraints into optimization problems. However, when precise geometric modeling is required, these collision-free constraints are generally non-convex, resulting in the need to solve non-convex optimization problems. Such problems are challenging to solve efficiently and reliably onboard, making these methods less suitable for safety-critical systems that demand real-time responsiveness.
CBF-based methods, on the other hand, achieve geometric-aware collision avoidance by transforming the collision-free constraints into linear constraints with respect to the control inputs. These linear constraints can be seamlessly incorporated into a quadratic program (QP) that minimally modifies a nominal controller. Since QPs are convex problems that can be solved efficiently and reliably, CBF-based methods have demonstrated their advantages in computational efficiency, particularly in the context of real-time safety-critical control.

To date, there are only a few CBF-based methods for geometric-aware collision avoidance. For instance, signed distance is utilized to design CBFs in \cite{singletary2022safety} for collision avoidance between general primitives. However, the evaluation of signed distance involves non-differentiable algorithms like the Gilbert-Johnson-Keerthi algorithm \cite{gilbert2002fast}, which makes the computation of the time derivatives of signed distance challenging. To circumvent this issue, the time derivatives of signed distance are approximated, resulting in a conservative controller. To eliminate the conservatism introduced by approximation, a duality-based CBF has been proposed in \cite{thirugnanam2022duality} for collision avoidance between polyhedra. Nonetheless, additional optimization problems and virtual states are required for the evaluation of the CBF, making the method less computationally efficient.
To address the non-smooth nature of signed distance, growth distance \cite{ong1996growth} is utilized to design smooth CBFs in \cite{dai2023safe} for collision avoidance between convex primitives. The time derivative of growth distance is calculated by leveraging the Karush-Kuhn-Tucker (KKT) conditions. Despite these advancements, two challenges still hinder the application of CBF-based methods in real-world scenarios.
Firstly, from a theoretical perspective, the validity of collision-free CBFs is questionable. For instance, the proposed CBF in \cite{thirugnanam2022duality} is not continuously differentiable, as shown in their simulation results, and the gradient of the CBF may vanish on the boundary of the safe set. These two features may affect the forward invariance properties of CBFs \cite{ames2019control}.
Secondly, from an implementation perspective, the evaluation of Euclidean distance, signed distance, and growth distance, along with their time derivatives, involves solving additional optimization problems. This results in a double-level optimization process, which reduces computational efficiency.

To address these challenges, this paper proposes a novel collision avoidance control method based on CBFs and the separating hyperplane theorem. The geometric shape of the controlled object is modeled as a general ellipsoid to balance theoretical rigor and computational efficiency. Since the geometry of most controlled objects, such as vehicles \cite{dai2024sailing}, manipulator components \cite{dai2023safe}, and quadcopters \cite{han2021fast}, can be approximated by general ellipsoids, this modeling approach can meet the precision requirements in most scenarios.
The main contributions of this paper are summarized as follows:
\begin{enumerate}
  \item The collision-free conditions for two general ellipsoids are derived analytically based on the separating hyperplane theorem and the concept of the dual cone.
  \item Leveraging these analytical collision-free conditions, novel collision-free CBFs are proposed, with their validity rigorously proven.
  \item A collision avoidance control method is developed that avoids the need for solving additional optimization problems, offering computational efficiency compared to double-level optimization methods.
  \item Simulations and experiments are conducted to demonstrate the effectiveness and practicality of the proposed collision avoidance control method.
\end{enumerate}

\section{Preliminaries}
\subsection{Notation}
Throughout this paper, $\mathbb{R}$ denotes the set of real numbers, $\mathbb{R}^n$ denotes the set of $n$-dimensional column vectors over $\mathbb{R}$, and $\mathbb{R}^{n \times m}$ denotes the set of $m$-by-$n$ matrices over $\mathbb{R}$.
Non-bold symbols are used for scalars $a \in \mathbb{R}$, bold lowercase symbols for vectors $\bm{a} \in \mathbb{R}^{n}$ and bold uppercase symbols for matrices $\bm{A} \in \mathbb{R}^{n \times m}$.
Specifically, $\bm{I}_d$ denotes the $d$-dimensional identity matrix, and
${SO}(d)$ denotes the $d$-dimensional special orthogonal group defined as ${SO}(d) = \{ \bm{R} \in \mathbb{R}^{d \times d} ~|~ \bm{R}^T \bm{R} = \bm{I}_d,~ \mathrm{det}(\bm{R}) = 1\}$.

Given a proper cone $\!\mathcal{K} \! \subseteq \! \mathbb{R}^{d}\!$, the general inequality $\!\preceq_{\mathcal{K}}\!$ is defined as,
\[\bm{x} \preceq_{\mathcal{K}} \bm{y} \quad \Longleftrightarrow \quad \bm{y}-\bm{x} \in \mathcal{K}.\]
Specifically, $\preceq$ denotes the inequalities introduced by nonnegative orthant cone $\mathbb{R}_{+}^{d}$, and $\bm{x} \preceq \bm{y}$ means $\bm{x}$ is component-wise less than or equal to $\bm{y}$.
Given a cone $\mathcal{K} \subset \mathbb{R}^{d}$, then the set
\begin{equation}
  \label{eqn:dual cone}
  \mathcal{K}^* = \big\{ \bm{y} \in \mathbb{R}^{d} ~\big|~ \bm{y}^T\bm{x} \geq 0, ~ \forall \bm{x} \in \mathcal{K} \big\},
\end{equation}
is called the dual cone of $\mathcal{K}$.

A function $\alpha: \mathbb{R} \rightarrow \mathbb{R}$ is said to be an \textit{extended class $\mathcal{K}_{\infty}$ function} if it is strictly increasing and $\alpha(0) = 0$.

\subsection{Control Barrier Function}
Consider the following control affine system,
\begin{equation}
  \label{eqn:control affine system}
  \dot{\bm{s}} = \bm{f}(\bm{s}) + \bm{g}(\bm{s}) \bm{u},
\end{equation}
where $\bm{s}\in \mathbb{R}^n$ and $\bm{u} \in \mathbb{R}^m$ are the state and control, respectively.
The functions $\bm{f}: \mathbb{R}^n \rightarrow \mathbb{R}^n$ and $\bm{g}: \mathbb{R}^n \rightarrow \mathbb{R}^{n\times m}$ are locally Lipschitz.
Let $h:\mathbb{R}^n \rightarrow \mathbb{R}$ be a continuously differentiable function, and define the \textit{safe set} $\mathcal{S}$ as the zero super-level set of $h$, that is,
\begin{equation}
  \label{eqn:safe set}
  \mathcal{S} = \{ \bm{s} ~|~ h(\bm{s}) \geq 0 \}.
\end{equation}
Then, a control barrier function can be defined as follows.
\begin{defn}[Control Barrier Functions \cite{ames2016control}]
  Let $\mathcal{S} \subset \mathcal{\overline{S}} \subset \mathbb{R}^n$ be the zero super-level set of a continuously differentiable function $h: \mathcal{\overline{S}} \rightarrow \mathbb{R}$. Then, $h$ is said to be a \textit{control barrier function} on $\mathcal{\overline{S}}$ if there exists an extended class $\mathcal{K}_{\infty}$ function $\alpha$ such that the time derivative of $h$ along the trajectory of system \eqref{eqn:control affine system} satisfies
  \begin{equation}
    \sup_{\bm{u} \in \mathbb{R}^m} \underbrace{{L}_{\bm{f}}h(\bm{s}) + {L}_{\bm{g}}h(\bm{s}) \bm{u}}_{\dot{h}(\bm{s},\bm{u})} \geq -\alpha(h(\bm{s})), \quad \forall \bm{s} \in \mathcal{\overline{S}},
  \end{equation}
  where ${L}_{\bm{f}}h(\cdot) = \frac{\partial h}{\partial \bm{s}}(\cdot)^T \bm{f}(\cdot)$ and ${L}_{\bm{g}}h(\cdot) = \frac{\partial h}{\partial \bm{s}}(\cdot)^T \bm{g}(\cdot)$ denote the Lie derivatives of $h$ with respect to $\bm{f}$ and $\bm{g}$, respectively.
\end{defn}
The existence of CBFs ensures the existence of controllers that guarantee the forward invariance of the safe set $\mathcal{S}$.
This property is formally stated in the following lemma.
\begin{lem}[\cite{ames2016control}]
  \label{lem:CBF}
  Let $\mathcal{S} \subset \mathcal{\overline{S}} \subset \mathbb{R}^n$ be a set defined as the zero super-level set of a function $h$ such that:
  \begin{enumerate}
    \item $h: \mathcal{\overline{S}} \subset \mathbb{R}^n \rightarrow \mathbb{R}$ is continuously differentiable on $\mathcal{\overline{S}}$,
    \item $h$ is a control barrier function on $\mathcal{\overline{S}}$,
    \item $\frac{\partial h}{\partial \bm{s}}(\bm{s}) \neq \bm{0}$ for all $\bm{s} \in \partial \mathcal{S}$.
  \end{enumerate}
  Define the following set induced by $h(\bm{s})$:
  \begin{equation}
    \label{eqn:control barrier function set}
    \mathcal{U}_{cbf}(\bm{s}) = \Big\{ \bm{u} ~|~ {L}_{\bm{f}}h(\bm{s}) + {L}_{\bm{g}}h(\bm{s}) \bm{u} \geq -\alpha(h(\bm{s})) \Big\}.
  \end{equation}
  Then, any Lipschitz continuous controller $\bm{u}(\bm{s}) \in \mathcal{U}_{cbf}(\bm{s})$ for system \eqref{eqn:control affine system} renders the set $\mathcal{S}$ forward invariant. Additionally, the set $\mathcal{S}$ is asymptotically stable in $\mathcal{\overline{S}}$.
\end{lem}

Given a desired controller $\bm{u}^d(\bm{s})$, the forward invariance of the safe set $\mathcal{S}$ can be achieved by minimally modifying $\bm{u}^d(\bm{s})$ such that the condition \eqref{eqn:control barrier function set} is satisfied.
Noting that the condition \eqref{eqn:control barrier function set} is a linear constraint on $\bm{u}$, the safe controller $\bm{u}^*$ can be computed via the following QP:
\begin{subequations}
  \begin{align}
    \bm{u}^* = \arg \min_{\bm{u}} & \qquad \frac{1}{2} \Vert \bm{u} - \bm{u}^d(\bm{s}) \Vert_2^2                        \\
    \text{s.t.} ~                 & \quad {L}_{\bm{f}}h(\bm{s}) + {L}_{\bm{g}}h(\bm{s}) \bm{u} \geq -\alpha(h(\bm{s})).
  \end{align}
\end{subequations}

\section{Collision-free Conditions for General Ellipsoids}
\label{sec:collision-free conditions}
In this section, novel collision-free conditions between two general ellipsoids are proposed based on separating hyperplanes.
The geometry of the $i$-th controlled object in the inertial frame is denoted as $\mathcal{G}_i$, which is defined as
\begin{equation}
  \label{eqn:geometries}
  \mathcal{G}_i = \{\bm{y} ~|~ \bm{y} = \bm{R}_i \bm{x} + \bm{\rho}_i, ~\bm{x} \in \mathcal{B}_i\},
\end{equation}
where $\mathcal{B}_i \subset \mathbb{R}^d$ is a compact convex set representing the geometry of the $i$-th controlled object in its body frame.
Moreover, $\bm{R}_i \in SO(d)$ and $\bm{\rho}_i \in \mathbb{R}^d$ are the rotation matrix and translation vector, respectively, that transform the body frame of object $i$ to the inertial frame.
Note that $\bm{R}_i$ and $\bm{\rho}_i$ should be regarded as the state of object $i$.
In this paper, $\mathcal{B}_i$ is assumed to be a general ellipsoid, which is defined as
\begin{equation}
  \label{eqn:general ellipsoid}
  \mathcal{B}_i = \left\{ \bm{x} \in \mathbb{R}^d ~\big|~ \Vert \bm{Q}_i^{-1} \bm{x} \Vert_{p_i}\leq 1 \right\},
\end{equation}
where $p_i > 1$ is the order of ellipsoid, and $\bm{Q}_i \in \mathbb{R}^{d \times d}$ is an invertible matrix.

To ensure the safety of objects $i$ and $j$, their geometries must not overlap, i.e., $\mathcal{G}_i \cap \mathcal{G}_j = \emptyset$.
This geometric constraint can be expressed mathematically using the following theorem:

\begin{thm}[Separating Hyperplane Theorem \cite{boyd2004convex}]
  \label{thm:separating hyperplane}
  Let $\mathcal{G}_i$ and $\mathcal{G}_j$ be two nonempty disjoint convex sets in $\mathbb{R}^d$, i.e., $\mathcal{G}_i \cap \mathcal{G}_j = \emptyset$.
  Then, there exist a normal vector $\bm{n}_{ij} \in \mathbb{R}^d$, $\bm{n}_{ij} \neq \bm{0}$, and an offset $\gamma_{ij} \in \mathbb{R}$ such that
  \begin{subequations}
    \label{eqn:separating hyperplane}
    \begin{gather}
      \bm{n}_{ij}^T \bm{y} \geq \gamma_{ij}, \quad \forall \bm{y} \in \mathcal{G}_i, \\
      \bm{n}_{ij}^T \bm{y} \leq \gamma_{ij}, \quad \forall \bm{y} \in \mathcal{G}_j.
    \end{gather}
  \end{subequations}
  The hyperplane $\{\bm{y} ~|~ \bm{n}_{ij}^T\bm{y} = \gamma_{ij}\}$ is called a separating hyperplane for the sets $\mathcal{G}_i$ and $\mathcal{G}_j$.
\end{thm}

\begin{figure}[htb]
  \vspace{-1.0em}
  \centering
  \includegraphics[width=0.8\hsize]{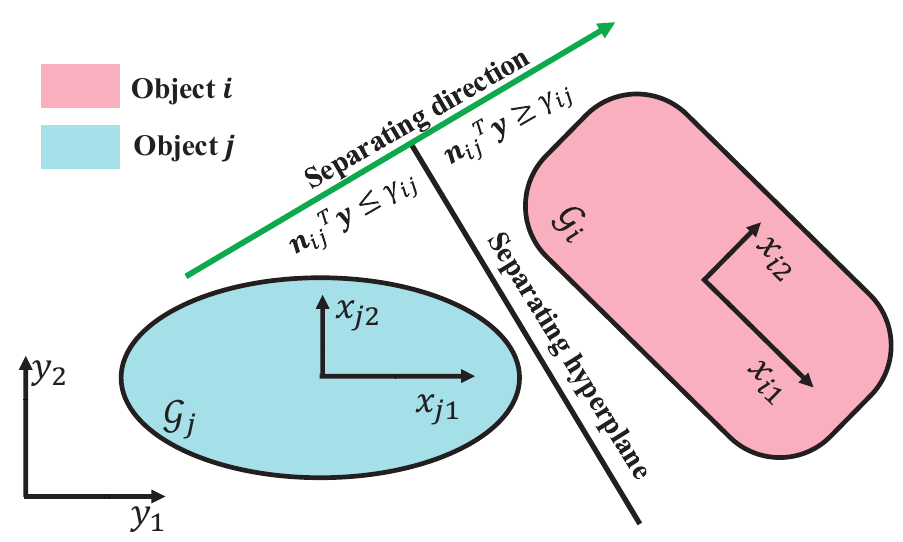}
  \caption{An illustration of the separating hyperplane theorem.}
  \label{fig:separating hyperplane}
  \vspace{-1.0em}
\end{figure}

The geometric illustration of Theorem \ref{thm:separating hyperplane} can be found in Fig. \ref{fig:separating hyperplane}.
Theorem \ref{thm:separating hyperplane} provides a key insight for collision avoidance between two general ellipsoids, namely, ensuring the existence of separating hyperplanes at all times.
However, due to the universal quantifier in constraint \eqref{eqn:separating hyperplane}, expressing the existence conditions of separating hyperplane $(\bm{n}_{ij},\gamma_{ij})$ as the zero super-level sets of continuously differentiable functions remains an open problem.
To address this gap, a new form of constraint \eqref{eqn:separating hyperplane} that without the universal quantifier is derived analytically using the dual cone of $p$-norm cone.
\begin{lem}[\cite{boyd2004convex}]
  \label{lem:norm cone}
  Given an order $p > 1$, the $p$-norm cone is defined as $\mathcal{K}_p = \{(\beta, \bm{z}) ~|~ \Vert \bm{z} \Vert_p \leq \beta \}$. Then its dual cone is $\mathcal{K}_p^* = \{(\lambda, \bm{\mu}) ~|~ \Vert \bm{\mu} \Vert_q \leq \lambda \}$ with $\frac{1}{q}+\frac{1}{p} = 1$.
\end{lem}
Rather than directly deriving the existence conditions of separating hyperplanes for two disjoint general ellipsoids, we first characterize the conditions for a hyperplane that ensures an entire ellipsoid on one side of it, as stated in the following lemma.

\begin{lem}
  \label{lem:single plane}
  Given a general ellipsoid $\mathcal{G}_i$ described by \eqref{eqn:geometries} and \eqref{eqn:general ellipsoid}, then the hyperplane
  \begin{subequations}
    \begin{equation}
      \mathcal{H} = \{\bm{y} ~|~ \bm{n}_{ij}^T \bm{y} = \gamma_{ij}\},\label{eqn:single hyperplane:plane}
    \end{equation}
    ensures entire $\mathcal{G}_i$ on one side of it in the following sense,
    \begin{equation}
      \bm{n}_{ij}^T \bm{y}_i \geq \gamma_{ij}, \quad \forall \bm{y}_i \in \mathcal{G}_i,\label{eqn:single hyperplane:separation}
    \end{equation}
    if and only if $\bm{n}_{ij} \neq \bm{0}$ and
    \begin{equation}
      \Vert (\bm{R}_i \bm{Q}_i)^T \bm{n}_{ij} \Vert_{q_i} \leq \bm{\rho}_i^T \bm{n}_{ij} - \gamma_{ij}, \label{eqn:single hyperplane:analytical}\\
    \end{equation}
  \end{subequations}
  with $\frac{1}{q_i}+\frac{1}{p_i} = 1$.
\end{lem}

\begin{proof}
  \textbf{Sufficiency ($\Rightarrow$):}
  Since $\mathcal{G}_i$ is a general ellipsoid and thus bounded, there always exists a hyperplane $\mathcal{H} = \{\bm{y} ~|~ \bm{n}_{ij}^T \bm{y} = \gamma_{ij}\}$ that satisfies condition \eqref{eqn:single hyperplane:plane}.
  The assumption that $\mathcal{H}$ is a hyperplane implies $\bm{n}_{ij} \neq \bm{0}$.
  Using the definition of the $p$-norm cone in Lemma \ref{lem:norm cone}, the ellipsoid $\mathcal{G}_i$ can be reformulated as
  \begin{subequations}
    \label{eqn:general ellipsoid:reformulation}
    \begin{align}
      \mathcal{G}_i & = \left\{ \bm{y} ~\big|~ \bm{y} = \bm{R}_i \bm{x} + \bm{\rho}_i, ~\Vert \bm{Q}_i^{-1} \bm{x} \Vert_{p_i} \leq 1 \right\}                                                   \\
                    & = \left\{ \bm{y} ~\big|~ \bm{y} = \bm{R}_i \bm{Q}_i \bm{z} + \bm{\rho}_i, ~\Vert \bm{z} \Vert_{p_i} \leq 1 \right\}                                                        \\
                    & = \left\{ \bm{y} ~\big|~ \bm{y} = \bm{R}_i \bm{Q}_i \bm{z} + \bm{\rho}_i, ~\forall (1, \bm{z}) \in \mathcal{K}_{p_i} \right\}. \label{eqn:general ellipsoid:def:norm cone}
    \end{align}
  \end{subequations}
  \begin{subequations}
    The condition \eqref{eqn:single hyperplane:separation} can then be reformulated as
    \begin{equation}
      \label{eqn:proof:single plane:a}
      \bm{n}_{ij}^T (\bm{R}_i \bm{Q}_i \bm{z}) + (\bm{\rho}_i^T \bm{n}_{ij} - \gamma_{ij}) \cdot 1 \geq 0, \quad \forall (1, \bm{z}) \in \mathcal{K}_{p_i}.
    \end{equation}
    Since $\mathcal{K}_{p_i}$ is a cone, $(1, \bm{z}) \in \mathcal{K}_{p_i}$ and $\beta \geq 0$ imply $(\beta, \beta \bm{z}) \in \mathcal{K}_{p_i}$.
    Consequently, multiplying both sides of \eqref{eqn:proof:single plane:a} by $\beta$, the following inequality is obtained
    \begin{equation}
      \bm{n}_{ij}^T (\bm{R}_i \bm{Q}_i \tilde{\bm{z}}) + (\bm{\rho}_i^T \bm{n}_{ij} - \gamma_{ij}) \cdot \beta \geq 0, \quad \forall (\beta, \tilde{\bm{z}}) \in \mathcal{K}_{p_i},
    \end{equation}
    where $\tilde{\bm{z}} = \beta \bm{z}$.
    By the definition of the dual cone \eqref{eqn:dual cone},
    \begin{equation}
      \big(\bm{\rho}_i^T \bm{n}_{ij} - \gamma_{ij},~ (\bm{R}_i \bm{Q}_i)^T \bm{n}_{ij} \big) \in \mathcal{K}_{p_i}^*.
    \end{equation}
    Finally, applying Lemma \ref{lem:norm cone}, the condition \eqref{eqn:single hyperplane:analytical} is obtained.
  \end{subequations}

  \textbf{Necessity ($\Leftarrow$):}
  \begin{subequations}
    If there exist a nonzero normal vector $\bm{n}_{ij}$ and an offset $\gamma_{ij}$ such that condition \eqref{eqn:single hyperplane:analytical} is satisfied, then Lemma \ref{lem:norm cone} implies
    \begin{equation}
      \Big(\bm{\rho}_i^T \bm{n}_{ij} - \gamma_{ij},~ (\bm{R}_i \bm{Q}_i)^T \bm{n}_{ij}\Big) \in \mathcal{K}_{p_i}^*.
    \end{equation}
    By the definition of dual cone, the following inequality holds:
    \begin{equation}
      \bm{n}_{ij}^T (\bm{R}_i \bm{Q}_i \bm{z}) + (\bm{\rho}_i^T \bm{n}_{ij} - \gamma_{ij}) \cdot 1 \geq 0, \quad \forall (1, \bm{z}) \in \mathcal{K}_{p_i}.
    \end{equation}
    Rearranging the above inequality, we have
    \begin{equation}
      \bm{n}_{ij}^T (\bm{R}_i \bm{Q}_i \bm{z} + \bm{\rho}_i) \geq \gamma_{ij}, \quad \forall (1, \bm{z}) \in \mathcal{K}_{p_i}.
    \end{equation}
    Noting that $(1, \bm{z}) \in \mathcal{K}_{p_i} \implies (\bm{R}_i \bm{Q}_i \bm{z} + \bm{\rho}_i) \in \mathcal{G}_i$ according to \eqref{eqn:general ellipsoid:def:norm cone}, we obtain the final inequality:
    \begin{equation}
      \bm{n}_{ij}^T \bm{y} \geq \gamma_{ij}, \quad \forall \bm{y} \in \mathcal{G}_i.
    \end{equation}
  \end{subequations}
  The proof is complete.
\end{proof}

Based on Lemma \ref{lem:single plane}, the existence conditions of separating hyperplanes for two disjoint general ellipsoids is characterized by the following theorem.
\begin{thm}
  \label{thm:feasible set of separating hyperplane}
  Let $\mathcal{G}_i$ and $\mathcal{G}_j$ be two disjoint general ellipsoids described by \eqref{eqn:geometries} and \eqref{eqn:general ellipsoid}.
  Then, $\mathcal{H} = \{ \bm{y} ~|~ \bm{n}_{ij}^T \bm{y} = \gamma_{ij} \}$ is a separating hyperplane for $\mathcal{G}_i$ and $\mathcal{G}_j$ in the following sense:
  \begin{equation}
    \label{eqn:separating hyperplane:new:sense}
    \bm{n}_{ij}^T \bm{y}_i \geq \gamma_{ij} \geq \bm{n}_{ij}^T \bm{y}_j, \quad \forall \bm{y}_i \in \mathcal{G}_i,~ \forall \bm{y}_j \in \mathcal{G}_j,
  \end{equation}
  if and only if there exists a nonzero normal vector $\bm{n}_{ij} \neq \bm{0}$ and an offset $\gamma_{ij}$ such that
  \begin{subequations}
    \label{eqn:separating hyperplane:new}
    \begin{gather}
      \Vert (\bm{R}_i \bm{Q}_i)^T \bm{n}_{ij} \Vert_{q_i} \leq \bm{\rho}_i^T \bm{n}_{ij} - \gamma_{ij}, \label{eqn:separating hyperplane:new:i} \\
      \Vert -(\bm{R}_j \bm{Q}_j)^T \bm{n}_{ij} \Vert_{q_j} \leq -\bm{\rho}_j^T \bm{n}_{ij} + \gamma_{ij}. \label{eqn:separating hyperplane:new:j}
    \end{gather}
  \end{subequations}
\end{thm}

\begin{proof}
  According to Lemma \ref{lem:single plane}, the first inequality in \eqref{eqn:separating hyperplane:new:sense} holds if and only if condition \eqref{eqn:separating hyperplane:new:i} is satisfied.
  It remains to prove the second inequality in \eqref{eqn:separating hyperplane:new:sense}, which is equivalent to the following inequality:
  \begin{subequations}
    \begin{equation}
      \label{eqn:separating hyperplane:sense:proof:j}
      - \bm{n}_{ij}^T \bm{y}_j \geq -\gamma_{ij}, \quad \forall \bm{y}_j \in \mathcal{G}_j.
    \end{equation}
    Applying Lemma \ref{lem:single plane}, the inequality \eqref{eqn:separating hyperplane:sense:proof:j} holds if and only if
    \begin{equation}
      \label{eqn:separating hyperplane:new:proof:j}
      \Vert - (\bm{R}_j \bm{Q}_j)^T \bm{n}_{ij} \Vert_{q_i} \leq -\bm{\rho}_i^T \bm{n}_{ij} + \gamma_{ij},
    \end{equation}
    which is equivalent to \eqref{eqn:separating hyperplane:new:j}.
  \end{subequations}
  The proof is complete.
\end{proof}

Theorem \ref{thm:feasible set of separating hyperplane} provides the analytical form of the feasible set of separating hyperplanes for two general ellipsoids.
By incorporating the constraints \eqref{eqn:separating hyperplane:new} into the CBFs framework, collision avoidance between two general ellipsoids can be achieved.
However, while the constraints \eqref{eqn:separating hyperplane:new} can be incorporated into the CBFs framework, ensuring $\bm{n}_{ij} \neq \bm{0}$ remains a challenge.
This challenge will be addressed in the next section.

\section{Collision Avoidance Control via CBFs With Separating Hyperplanes}
\label{sec:collision-free cbfs}
In this section, a novel collision avoidance control method is proposed by incorporating the collision-free conditions \eqref{eqn:separating hyperplane:new} into the CBF framework.

The object $i$ is assumed to be fully controlled, more precisely, the dynamics of rotation matrix $\bm{R}_i$ and translation vector $\bm{\rho}_i$ is described by,
\begin{equation}
  \label{eqn:object dynamics}
  \dot{\bm{R}}_i = \bm{R}_i \bm{\widehat{\omega}}_i, \quad \dot{\bm{\rho}}_i = \bm{R}_i \bm{v}_i,
\end{equation}
where $\bm{\omega}_i \in \mathbb{R}^{\frac{d(d-1)}{2}}$ and $\bm{v}_i \in \mathbb{R}^d$ are the angular velocity and translational velocity of object $i$ in its body frame.
And $\wedge$ is the operator that maps a vector to a screw symmetric matrix, for examples, given $\omega \in \mathbb{R}$ and $\bm{\omega} = (\omega_1,\omega_2,\omega_3) \in \mathbb{R}^3$, the corresponding operator $\wedge$ is defined as,
\begin{equation}
  \label{eqn:screw symmetric operator}
  \widehat{\omega} = \begin{bmatrix} 0& -\omega \\ \omega &0 \end{bmatrix}, \quad
  \widehat{\bm{\omega}} =
  \begin{bmatrix}
    0         & -\omega_3 & \omega_2  \\
    \omega_3  & 0         & -\omega_1 \\
    -\omega_2 & \omega_1  & 0
  \end{bmatrix}.
\end{equation}

To evaluate the collision-free conditions \eqref{eqn:separating hyperplane:new} between two convex primitives, a separating hyperplane is required.
A natural approach is to extend the state of the system with a nonzero normal vector $\bm{n}_{ij} \in \mathbb{R}^{d}$ and an offset $\gamma_{ij} \in \mathbb{R}$.
However, fulfilling the requirement $\bm{n}_{ij} \neq \bm{0}$ without introducing conservatism remains a challenge.
In this paper, this challenge is addressed by enforcing the normal vector to be a unit vector, i.e., $\bm{n}_{ij}^T\bm{n}_{ij} = 1$.
This approach does not introduce any conservatism because the hyperplane $\{\bm{y} ~|~ \bm{n}_{ij}^T \bm{y} = \gamma_{ij}\}$ is equivalent to $\{\bm{y} ~|~ \widehat{\bm{n}}_{ij}^T \bm{y} = \widehat{b}_{ij}\}$, where $\widehat{\bm{n}}_{ij} = {\bm{n}_{ij}}/{\Vert \bm{n}_{ij} \Vert_2}$ and $\widehat{b}_{ij} = {\gamma_{ij}}/{\Vert \bm{n}_{ij} \Vert_2}$.

To preserve the norm of $\bm{n}_{ij}$, the dynamics of the normal vector $\bm{n}_{ij}$ and the offset $\gamma_{ij}$ are designed as
\begin{equation}
  \label{eqn:hyperplane dynamics}
  \dot{\bm{n}}_{ij} = (\bm{I}_d - \bm{n}_{ij}\bm{n}_{ij}^T) \bm{\eta}_{ij}, \quad \dot{\gamma}_{ij} = \delta_{ij},
\end{equation}
where $\bm{\eta}_{ij} \in \mathbb{R}^d$ and $\delta_{ij} \in \mathbb{R}$ are the control inputs for $\bm{n}_{ij}$ and $\gamma_{ij}$, respectively.
The above dynamics ensure that the norm of $\bm{n}_{ij}$ remains unchanged, as the time derivative of $\bm{n}_{ij}^T\bm{n}_{ij}$ satisfies
\begin{equation}
  \begin{aligned}
    2 \bm{n}_{ij}^T\dot{\bm{n}}_{ij}
     & = 2 \bm{n}_{ij}^T(\bm{I}_d - \bm{n}_{ij}\bm{n}_{ij}^T) \bm{\eta}_{ij} \\
     & = 2(\bm{n}_{ij}^T - \bm{n}_{ij}^T) \bm{\eta}_{ij} = 0.
  \end{aligned}
\end{equation}
For notational simplicity and clarity, intermediate variables $\lambda_i, \lambda_j \in \mathbb{R}$ and $\bm{\mu}_i, \bm{\mu}_j \in \mathbb{R}^d$ are defined as follows:
\begin{subequations}
  \label{eqn:intermediate variables}
  \begin{gather}
    \lambda_i = \bm{\rho}_i^T \bm{n}_{ij} - \gamma_{ij}, \quad \bm{\mu}_i = (\bm{R}_i \bm{Q}_i)^T \bm{n}_{ij}, \\
    \lambda_j = -\bm{\rho}_j^T \bm{n}_{ij} + \gamma_{ij}, \quad \bm{\mu}_j = -(\bm{R}_j \bm{Q}_j)^T \bm{n}_{ij}.
  \end{gather}
\end{subequations}
Following the discussion in Section \ref{sec:collision-free conditions}, the collision-free conditions \eqref{eqn:separating hyperplane:new} can be expressed as
\begin{subequations}
  \label{eqn:collision-free CBF}
  \begin{gather}
    h_i(\lambda_i, \bm{\mu}_i) = \lambda_i - \Vert \bm{\mu}_i \Vert_{q_i} \geq \bm{0}, \label{eqn:collision-free CBF:i}\\
    h_j(\lambda_j, \bm{\mu}_j) = \lambda_j - \Vert \bm{\mu}_j \Vert_{q_j} \geq \bm{0}. \label{eqn:collision-free CBF:j}
  \end{gather}
\end{subequations}
To incorporate the above constraints into the CBFs framework, the time derivatives of functions $h_i$ and $h_j$ are required, which are given by the following lemma.
\begin{lem}
  \label{lem:time derivatives of CBFs}
  The time derivatives of functions $h_i(\lambda_i, \bm{\mu}_i)$ and $h_j(\lambda_j, \bm{\mu}_j)$ along the system dynamics \eqref{eqn:object dynamics}, \eqref{eqn:hyperplane dynamics} and \eqref{eqn:intermediate variables} are derived as,
  \begin{subequations}
    \label{eqn:time derivatives of CBFs}
    \begin{gather}
      \dot{h}_i = \bm{a}_i^T \bm{\omega}_i + \bm{b}_i^T \bm{v}_i + \bm{c}_i^T \bm{\eta}_{ij} + \bm{d}_i^T \delta_{ij}, \\
      \dot{h}_j = \bm{a}_j^T \bm{\omega}_j + \bm{b}_j^T \bm{v}_j + \bm{c}_j^T \bm{\eta}_{ij} + \bm{d}_j^T \delta_{ij},
    \end{gather}
  \end{subequations}
  where the coefficient vectors are detailed as
  \begin{subequations}
    \label{eqn:time derivatives of CBFs:BCD}
    \begin{gather*}
      \bm{b}_i = \bm{R}_i^T \bm{n}_{ij} \frac{\partial h_i}{\partial \lambda_i}  , \quad \bm{d}_i =  - \frac{\partial h_i}{\partial \lambda_i} , \\
      \bm{c}_i = (\bm{I}_d - \bm{n}_{ij}\bm{n}_{ij}^T) \left(\bm{\rho}_i \frac{\partial h_i}{\partial \lambda_i} + \bm{R}_i \bm{Q}_i \frac{\partial h_i}{\partial \bm{\mu}_i} \right),\\
      \bm{b}_j = -\bm{R}_j^T \bm{n}_{ij} \frac{\partial h_j}{\partial \lambda_j}  , \quad \bm{d}_j = \frac{\partial h_j}{\partial \lambda_j} , \\
      \bm{c}_j = - (\bm{I}_d - \bm{n}_{ij}\bm{n}_{ij}^T) \left(\bm{\rho}_j \frac{\partial h_j}{\partial \lambda_j} + \bm{R}_j \bm{Q}_j \frac{\partial h_j}{\partial \bm{\mu}_j}\right).
    \end{gather*}
  \end{subequations}
  For $d\!=\!2$, the coefficient vectors $\bm{a}_i$ and $\bm{a}_j$ are given as
  \begin{subequations}
    \begin{gather}
      \label{eqn:time derivatives of CBFs:A:2}
      \bm{a}_i = \bm{n}_{ij}^T \bm{R}_i \widehat{1} \bm{Q}_i \frac{\partial h_i}{\partial \bm{\mu}_i} , \\
      \bm{a}_j = - \bm{n}_{ij}^T \bm{R}_j \widehat{1} \bm{Q}_j \frac{\partial h_j}{\partial \bm{\mu}_j}.
    \end{gather}
  \end{subequations}
  For $d \!=\!3$, the coefficient vectors $\bm{a}_i$ and $\bm{a}_j$ are given as
  \begin{subequations}
    \begin{gather}
      \label{eqn:time derivatives of CBFs:A:3}
      \bm{a}_i = \widehat{(\bm{R}_i^T \bm{n}_{ij})}^T \bm{Q}_i \frac{\partial h_i}{\partial \bm{\mu}_i}, \\
      \bm{a}_j = -\widehat{(\bm{R}_j^T \bm{n}_{ij})}^T \bm{Q}_j \frac{\partial h_j}{\partial \bm{\mu}_j}.
    \end{gather}
  \end{subequations}
\end{lem}
\begin{proof}
  The proof is given in Appendix \ref{app:proof:lem:time derivatives of CBFs}.
\end{proof}
Note that, the differences between \eqref{eqn:time derivatives of CBFs:A:2} and \eqref{eqn:time derivatives of CBFs:A:3} are caused by the slightly differences of the $\wedge$ operator between $d=2$ and $d=3$ as shown in \eqref{eqn:screw symmetric operator}.

To establish that the collision-free CBFs \eqref{eqn:collision-free CBF} constitute valid control barrier functions, we must verify that their coefficient vectors are non-zero and continuously differentiable with respect to the system state. These properties are critical to ensure the CBFs are well-defined and can reliably enforce safety constraints within the control framework. A detailed discussion of their necessity is provided in \cite[Remark 5]{ames2019control}.
The validity of the proposed CBFs is formally established by the following theorem.

\begin{thm}
  \label{thm:validity of CBFs}
  Define the safe set $\mathcal{S}_{ij}$ as
  $$
    \mathcal{S}_{ij} = \left\{ (\bm{R}_i,\bm{\rho}_i,\bm{R}_j,\bm{\rho}_j, \bm{n}_{ij}, \gamma_{ij}) \,\bigg|\, h_i \geq 0,\ h_j \geq 0 \right\}.
  $$
  Then, $h_i$ and $h_j$ are valid CBFs for the dynamics \eqref{eqn:object dynamics} and \eqref{eqn:hyperplane dynamics} on $\mathcal{S}_{ij}$, as formalized below:
  \begin{enumerate}
    \item The coefficient vectors $(\bm{a}_i, \bm{b}_i, \bm{c}_i, \bm{d}_i)$ and $(\bm{a}_j, \bm{b}_j, \bm{c}_j,\bm{d}_j)$ given by Lemma \ref{lem:time derivatives of CBFs} are continuous.
    \item The coefficient vectors $(\bm{a}_i, \bm{b}_i, \bm{c}_i, \bm{d}_i)$ and $(\bm{a}_j, \bm{b}_j, \bm{c}_j,\bm{d}_j)$ will not vanish, i.e., become zero vectors.
    \item There exists at least one control input that simultaneously satisfies $\dot{h}_i \geq -\alpha(h_i)$ and $\dot{h}_j \geq -\alpha(h_j)$.
  \end{enumerate}
\end{thm}

\begin{proof}
  The proof is provided in Appendix \ref{app:proof:thm:validity of CBFs}.
\end{proof}

The above theorem also ensures that, for any states inside the safe set $\mathcal{S}_{ij}$, at least one control input can be found such that renders the set $\mathcal{S}_{ij}$ forward invariant.

Given the nominal controller $\bm{u}_{ij}^d = (\bm{\omega}_i^d, \bm{v}_i^d, \bm{\omega}_j^d, \bm{v}_j^d)$ for two controlled objects, the collision-free control input $\bm{u}_{ij}^*$ for the controlled objects and the control input $(\bm{\eta}_{ij}^*, \delta_{ij}^*)$ for the separating hyperplane can be obtained by solving the following QP:

\begin{subequations}
  \begin{align}
    \min_{\bm{u}_{ij}, \bm{\eta}_{ij}, \delta_{ij}} & \qquad \frac{1}{2} \Vert \bm{u}_{ij} - \bm{u}_{ij}^d \Vert_2^2 \label{eqn:collision avoidance CBF-QP:origin:objective} \\
    \mathrm{s.t.} \quad                             & \qquad \dot{h}_i \geq - \alpha (h_i), \label{eqn:collision avoidance CBF-QP:origin:constraint:i}                       \\
                                                    & \qquad \dot{h}_j \geq - \alpha (h_j), \label{eqn:collision avoidance CBF-QP:origin:constraint:j}
  \end{align}
  \label{eqn:collision avoidance CBF-QP:origin}
\end{subequations}
where $\dot{h}_i$ and $\dot{h}_j$ are given by \eqref{eqn:time derivatives of CBFs}.

It is worth noting that the control input $(\bm{\eta}_{ij}, \delta_{ij})$ for the separating hyperplane is not incorporated into the objective function of the optimization problem \eqref{eqn:collision avoidance CBF-QP:origin}. This is because the normal vector $\bm{n}_{ij}$ and offset $\gamma_{ij}$ are virtual states that should adapt passively according to collision avoidance demands. Since the objective function in the optimization problem \eqref{eqn:collision avoidance CBF-QP:origin} is positive semi-definite, there may be infinitely many optimal solutions for $(\bm{\eta}_{ij}, \delta_{ij})$. This is not problematic, as all such control inputs ensure collision-free behavior. Moreover, the Lipschitz continuity of the control input can be guaranteed by selecting the optimal solution with the minimum norm.

\begin{figure*}[t]
  \captionsetup[subfloat]{farskip=2pt,captionskip=1pt}
  \centering
  \subfloat[$t = 0$ s.]{
    \includegraphics[clip, trim={1.2cm 1.1cm 1.8cm 2.0cm}, width=0.235\linewidth]{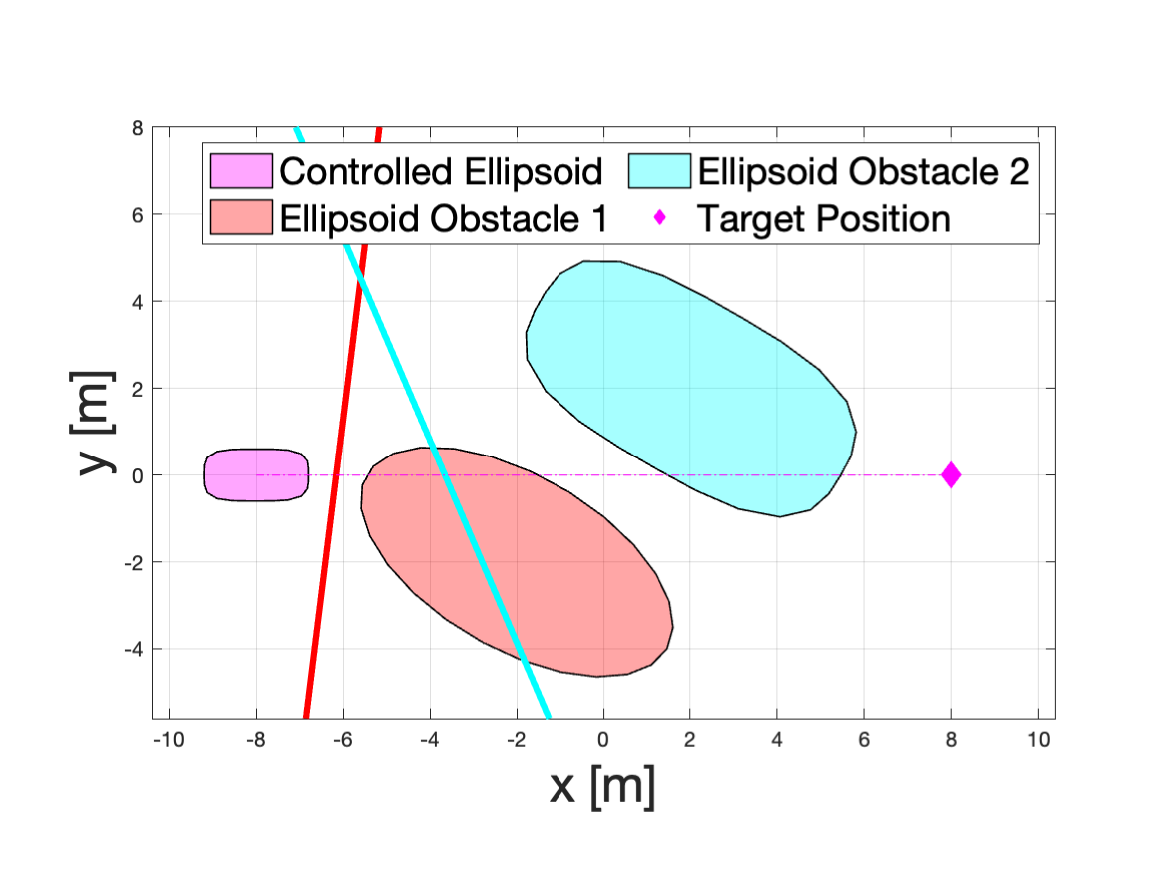}
    \label{fig:single vehicle simulation:1}} \hfill
  \subfloat[$t = 2$ s.]{
    \includegraphics[clip, trim={1.2cm 1.1cm 1.8cm 2.0cm}, width=0.235\linewidth]{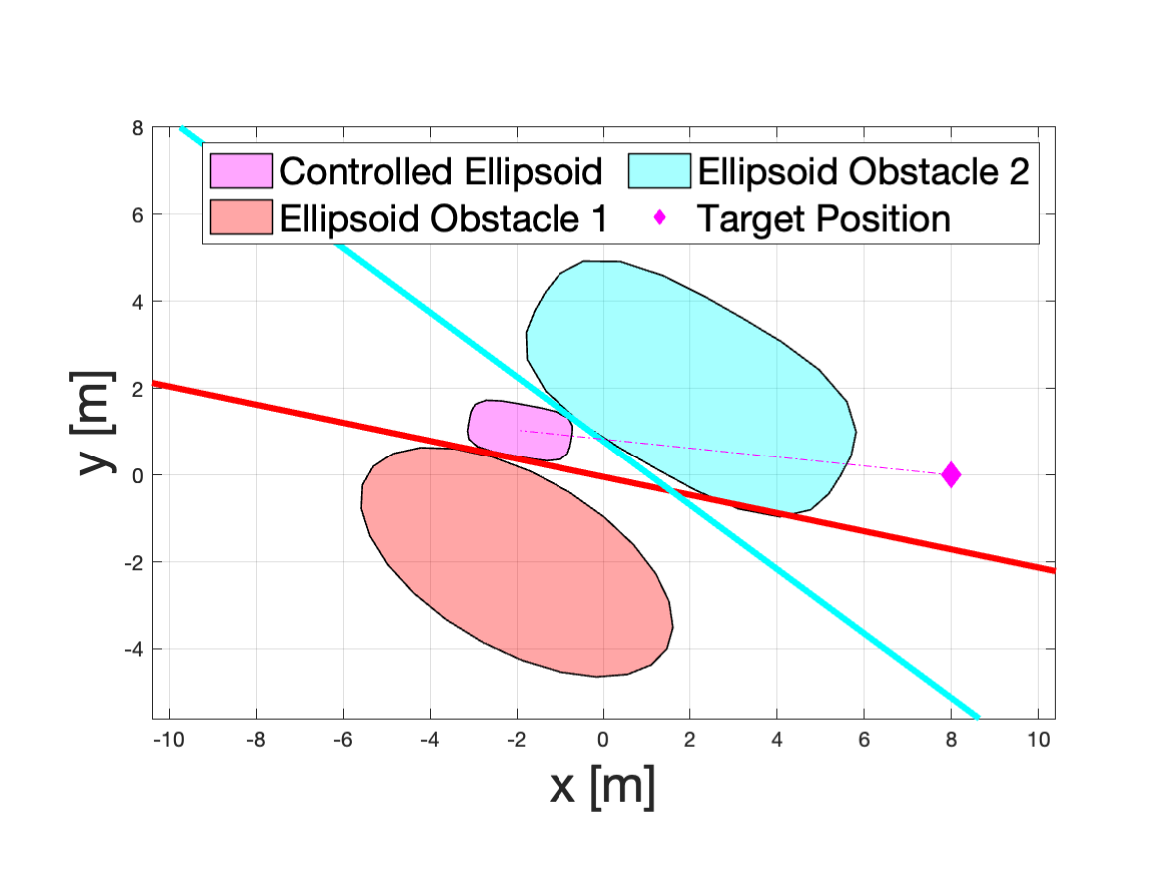}
    \label{fig:single vehicle simulation:2}}\hfill
  \subfloat[$t = 4$ s.]{
    \includegraphics[clip, trim={1.2cm 1.1cm 1.8cm 2.0cm}, width=0.235\linewidth]{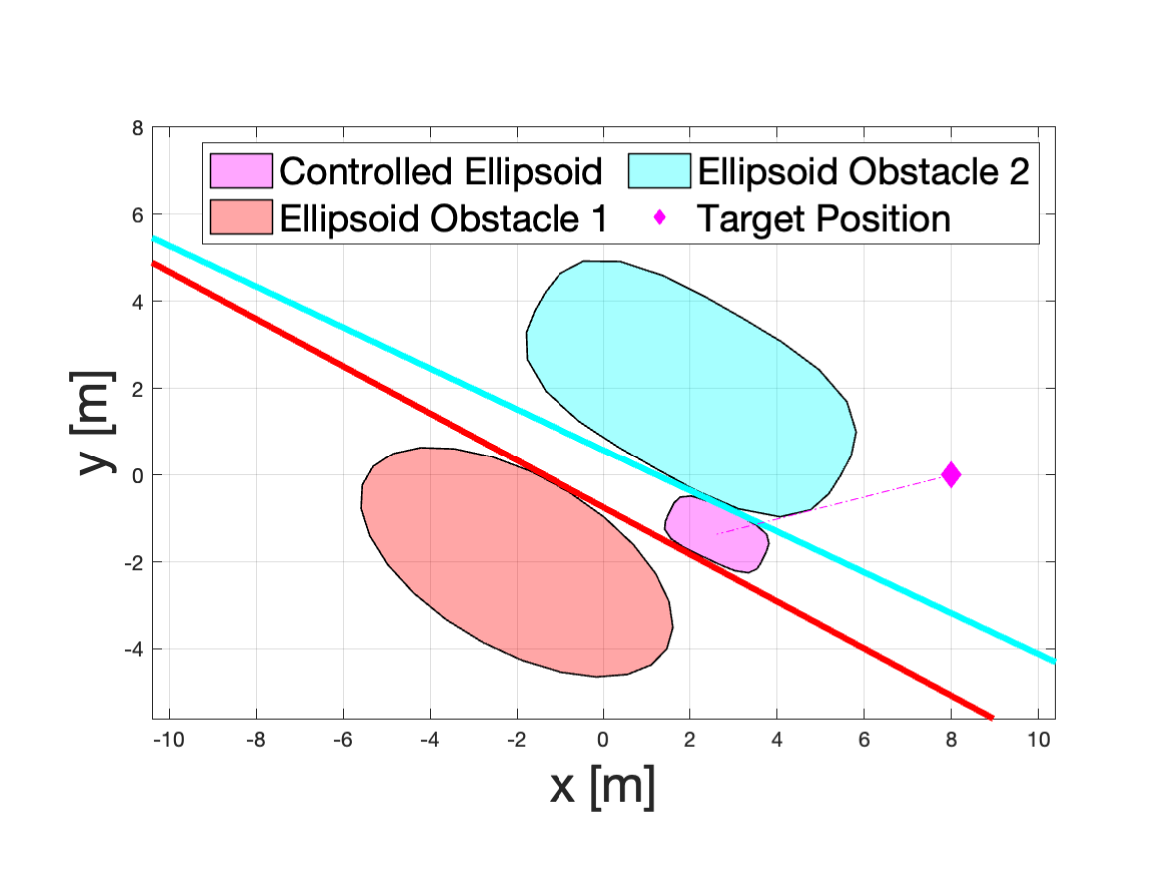}
    \label{fig:single vehicle simulation:3}} \hfill
  \subfloat[$t = 10$ s.]{
    \includegraphics[clip, trim={1.2cm 1.1cm 1.8cm 2.0cm}, width=0.235\linewidth]{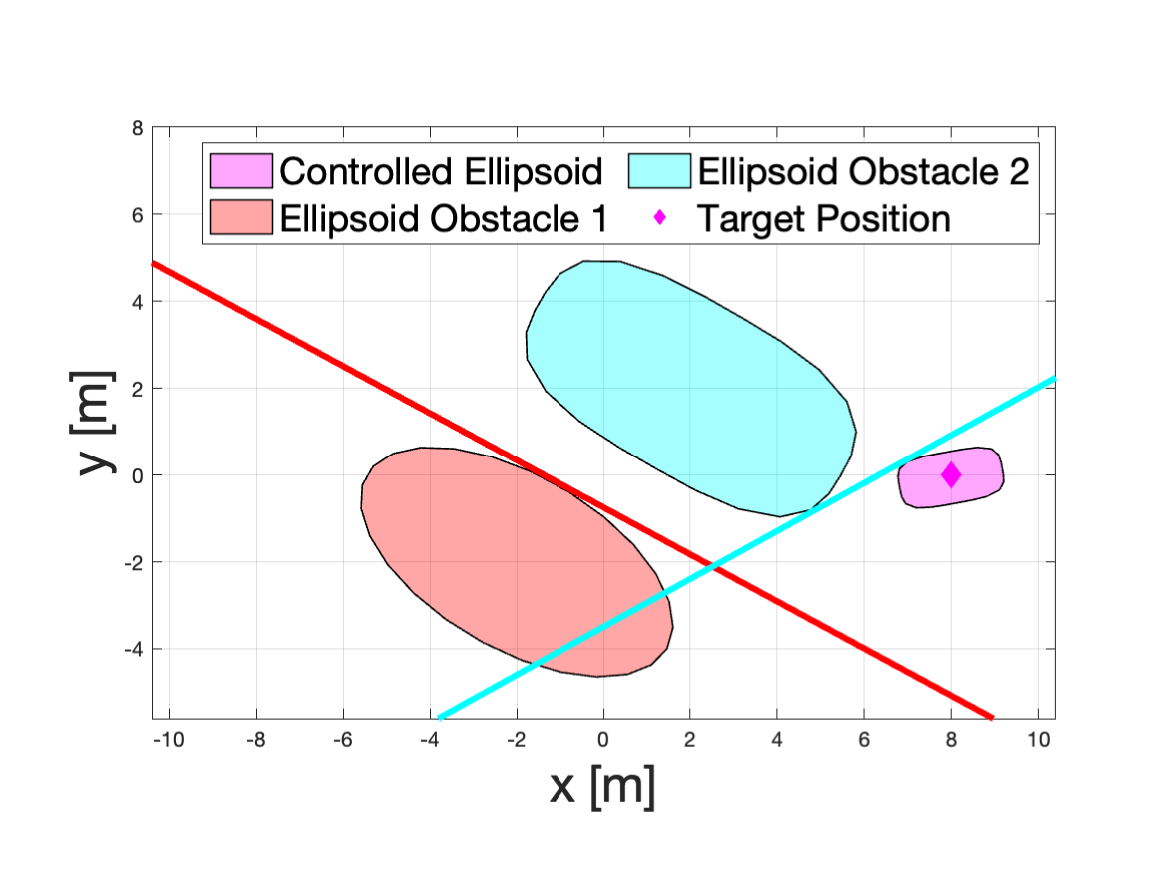}
    \label{fig:single vehicle simulation:4}}
  \caption{The snapshots of the simulation with the proposed collision avoidance method. The controlled ellipsoid navigates to its target position while avoiding collisions with the two ellipsoid obstacles.}
  \label{fig:single vehicle simulation}
  \vspace{-1em}
\end{figure*}

To ensure effective collision avoidance, it is essential to guarantee that the hyperplane $(\bm{n}_{ij}, \gamma_{ij})$ constitutes a valid separating hyperplane for the two objects at the initial time.
In the absence of such validity, the hyperplane may fail to properly separate the objects, thereby resulting in a violation of the collision avoidance constraints.
Furthermore, to prevent overly conservative evasion maneuvers, it is imperative to ensure that the initial hyperplane maintains sufficient separation from both objects.
These dual requirements can be satisfied by determining the maximum separating hyperplane for the two objects at the initial time.
For two disjoint general ellipsoids, the maximum separating hyperplane can be computed by solving the following optimization problem:
\begin{subequations}
  \begin{align}
    \min_{\widetilde{\bm{n}}_{ij}, \widetilde{\gamma}_{ij}} ~~ & \qquad  \qquad \qquad \widetilde{\bm{n}}_{ij}^T \widetilde{\bm{n}}_{ij} \label{eqn:collision avoidance CBF-QP:initial:objective}                                        \\
    \mathrm{s.t.} \quad                                        & \Vert (\bm{R}_i \bm{Q}_i)^T \widetilde{\bm{n}}_{ij} \Vert_{q_i} \leq \bm{\rho}_i^T \widetilde{\bm{n}}_{ij} - \widetilde{\gamma}_{ij} - 1, \label{eqn:collision avoidance CBF-QP:initial:constraint:i}   \\
                                                               & \Vert -(\bm{R}_j \bm{Q}_j)^T \widetilde{\bm{n}}_{ij} \Vert_{q_j} \leq -\bm{\rho}_j^T \widetilde{\bm{n}}_{ij} + \widetilde{\gamma}_{ij}, \label{eqn:collision avoidance CBF-QP:initial:constraint:j}
  \end{align}
  \label{eqn:collision avoidance CBF-QP:initial}
\end{subequations}
where the constraints \eqref{eqn:collision avoidance CBF-QP:initial:constraint:i} and \eqref{eqn:collision avoidance CBF-QP:initial:constraint:j} guarantee the following inequalities for any point $\bm{y}_i \in \mathcal{G}_i$ and $\bm{y}_j \in \mathcal{G}_j$:
\begin{equation}
  \widetilde{\bm{n}}_{ij}^T \bm{y}_i \geq \widetilde{\gamma}_{ij} + 1 > \widetilde{\gamma}_{ij} \geq \widetilde{\bm{n}}_{ij}^T \bm{y}_j.
\end{equation}
Consequently, the distance between the two objects is lower bounded by $1/ \Vert  \widetilde{\bm{n}}_{ij} \Vert_2$.
The objective function \eqref{eqn:collision avoidance CBF-QP:initial:objective} minimizes the Euclidean norm of the normal vector $\widetilde{\bm{n}}_{ij}$, which equivalently maximizes the separation distance between the hyperplane and both objects.
As a result, a valid separating hyperplane for collision avoidance is obtained as:
\begin{equation}
  \bm{n}_{ij} = \frac{\widetilde{\bm{n}}_{ij}}{\Vert  \widetilde{\bm{n}}_{ij} \Vert_2}, \quad \gamma_{ij} = \frac{ 2 \cdot \widetilde{\gamma}_{ij} + 1}{2  \cdot \Vert \widetilde{\bm{n}}_{ij} \Vert_2}.
\end{equation}
It is noteworthy that the optimization problem \eqref{eqn:collision avoidance CBF-QP:initial} is convex and can be solved efficiently using off-the-shelf solvers such as SCS \cite{odonoghue21operator}, CVX \cite{CVXtoolbox2008Michael}, and Mosek \cite{mosek}.

\section{Validation and Comparisons}

In this section, numerical simulations and real-world experiments are conducted to demonstrate the effectiveness of the proposed algorithm. The proposed algorithm is also compared with the following state-of-the-art collision avoidance algorithms based on CBFs.

\subsection{Simulation}

Numerical simulations are conducted to verify the efficacy of the proposed collision avoidance control method. The simulation scenario involves one fully controlled general ellipsoid $\mathcal{G}_0$, whose dynamics follow \eqref{eqn:object dynamics}. The initial position of $\mathcal{G}_0$ is $[-8; 0]$, and its initial rotation matrix is $[1, 0; 0, 1]$. The geometric parameters of $\mathcal{G}_0$ are set as $\bm{Q}_0 = [1.2, 0; 0, 0.6]$ and $p_0 = 4$. Additionally, two ellipsoid obstacles, $\mathcal{G}_1$ and $\mathcal{G}_2$, are centered at $\bm{\rho}_1 = [-2; -2]$ and $\bm{\rho}_2 = [2; 2]$, respectively. Their rotation matrices are $\bm{R}_1 = \bm{R}_2 = [\cos(\pi/6), \sin(\pi/6); -\sin(\pi/6), \cos(\pi/6)]$, with parameter matrices $\bm{Q}_1 = \bm{Q}_2 = [4, 0; 0, 2]$ and orders $p_1 = 2$ and $p_2 = 3$. The control objective is to drive $\mathcal{G}_0$ to the target position $\bm{\rho}_0^d = [8; 0]$ while avoiding collisions with the obstacles. To achieve this, the nominal controllers for $\mathcal{G}_0$ are designed as $\bm{\omega}_0^d = \bm{0}$ and $\bm{v}_0^d = -k_{\rho} \bm{R}_0^T (\bm{\rho}_0 - \bm{\rho}_0^d)$, where $k_{\rho} = 0.3$. The class $\mathcal{K}_{\infty}$ function is designed as $\alpha(h) = 20 h$.

Snapshots of the simulation are shown in Fig. \ref{fig:single vehicle simulation}. The two lines represent the separating hyperplanes between $\mathcal{G}_0$ and the two static ellipsoids, with each hyperplane colored to match its corresponding obstacle. The hyperplanes adapt passively to the motion of $\mathcal{G}_0$. As demonstrated in the snapshots, the hyperplanes consistently separate $\mathcal{G}_0$ from the obstacles, ensuring collision-free navigation.

\begin{figure}[h]
  \vspace{-0.7em}
  \centering
  \includegraphics[width=0.55\hsize]{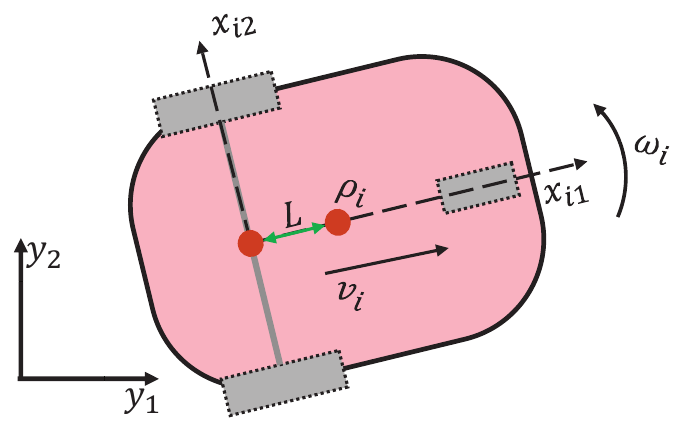}
  \caption{The nonholonomic vehicle model. The control inputs are the angular velocity $\omega_i$ and the forward translational velocity of the rear axle $v_i$ in the body frame of the vehicle.}
  \label{fig:nonholonomic}
  \vspace{-0.7em}
\end{figure}

\subsection{Experiment}

\begin{figure*}[!h]
  \centering
  \subfloat[$t=0$s.]{
    \includegraphics[clip, trim={5.5cm 1.5cm 6.5cm 2.0cm}, width=0.231\linewidth]{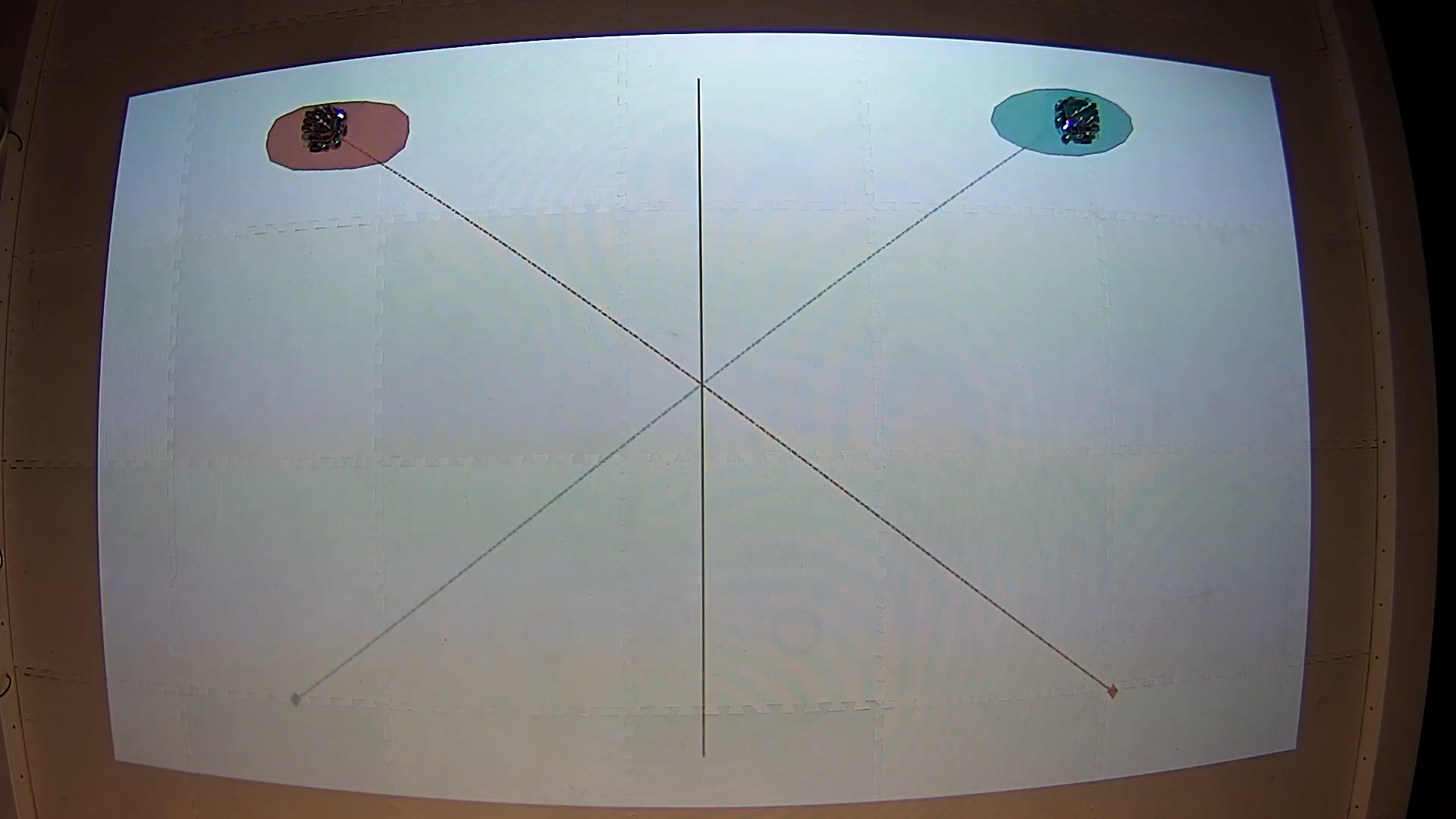}
    \label{fig:experiment:1}}
  \subfloat[$t=10$s.]{
    \includegraphics[clip, trim={5.5cm 1.5cm 6.5cm 2.0cm}, width=0.231\linewidth]{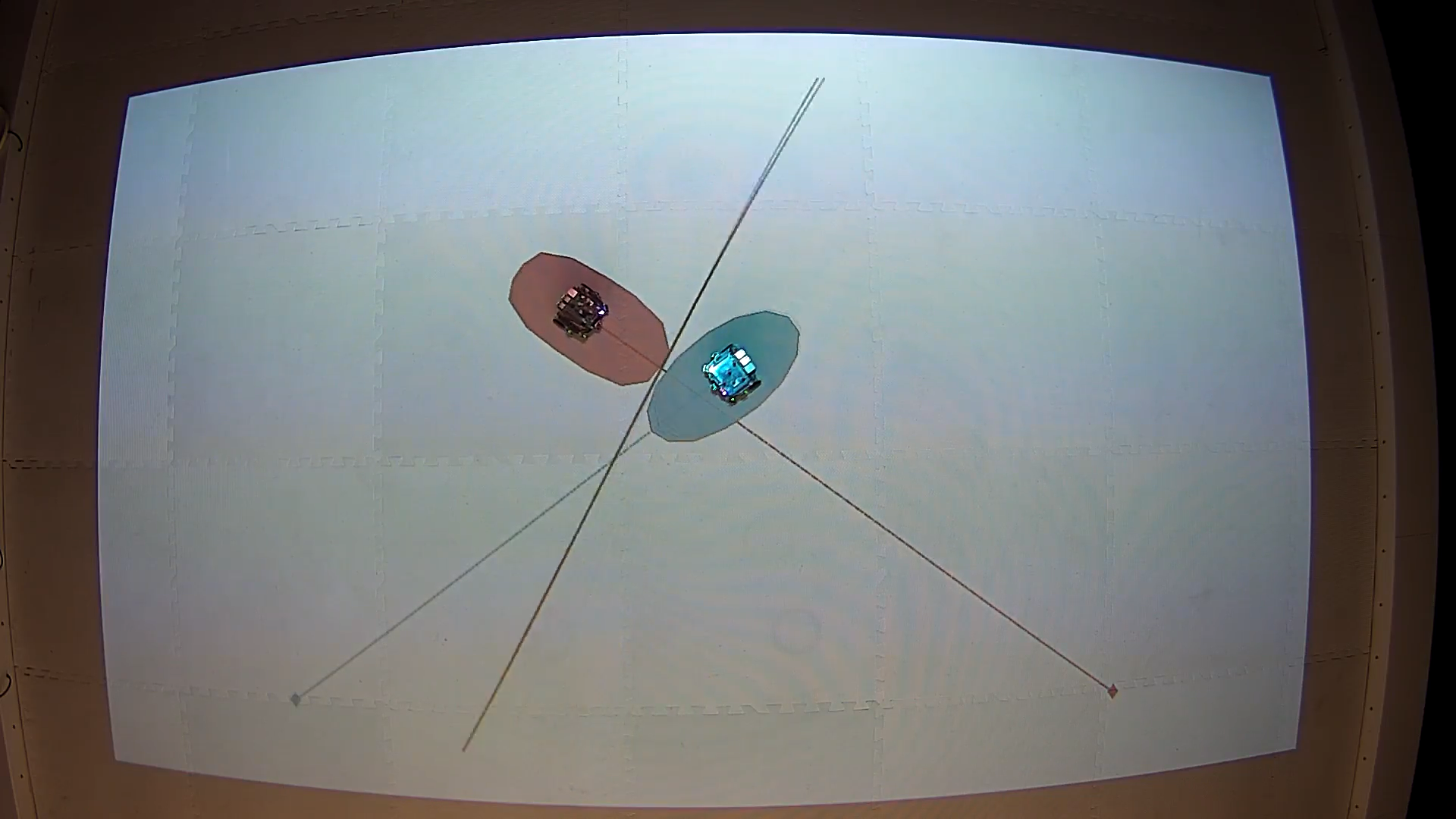}
    \label{fig:experiment:2}}
  \subfloat[$t=15$s.]{
    \includegraphics[clip, trim={5.5cm 1.5cm 6.5cm 2.0cm}, width=0.231\linewidth]{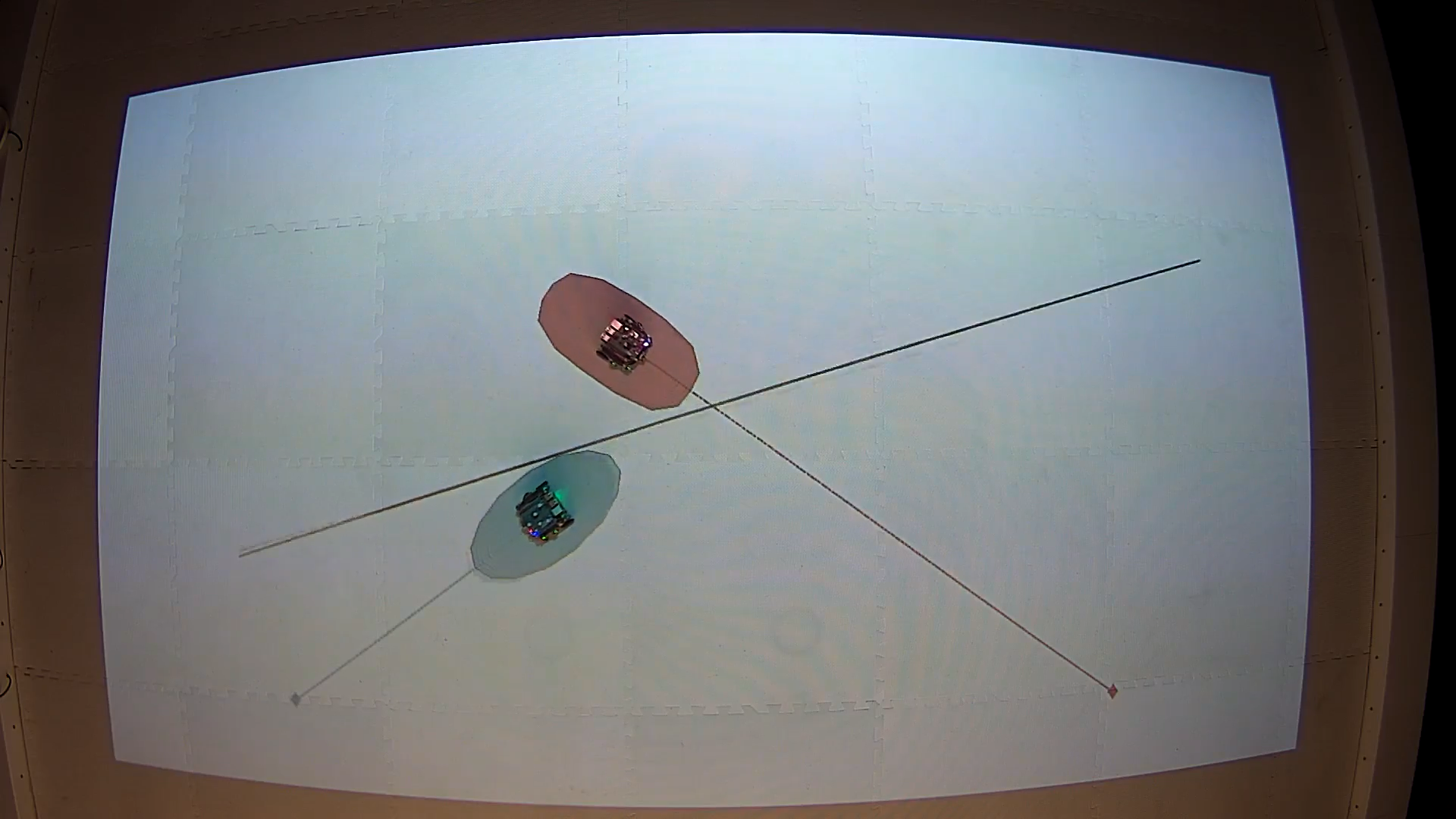}
    \label{fig:experiment:3}}
  \subfloat[$t=24$s.]{
    \includegraphics[clip, trim={5.5cm 1.5cm 6.5cm 2.0cm}, width=0.231\linewidth]{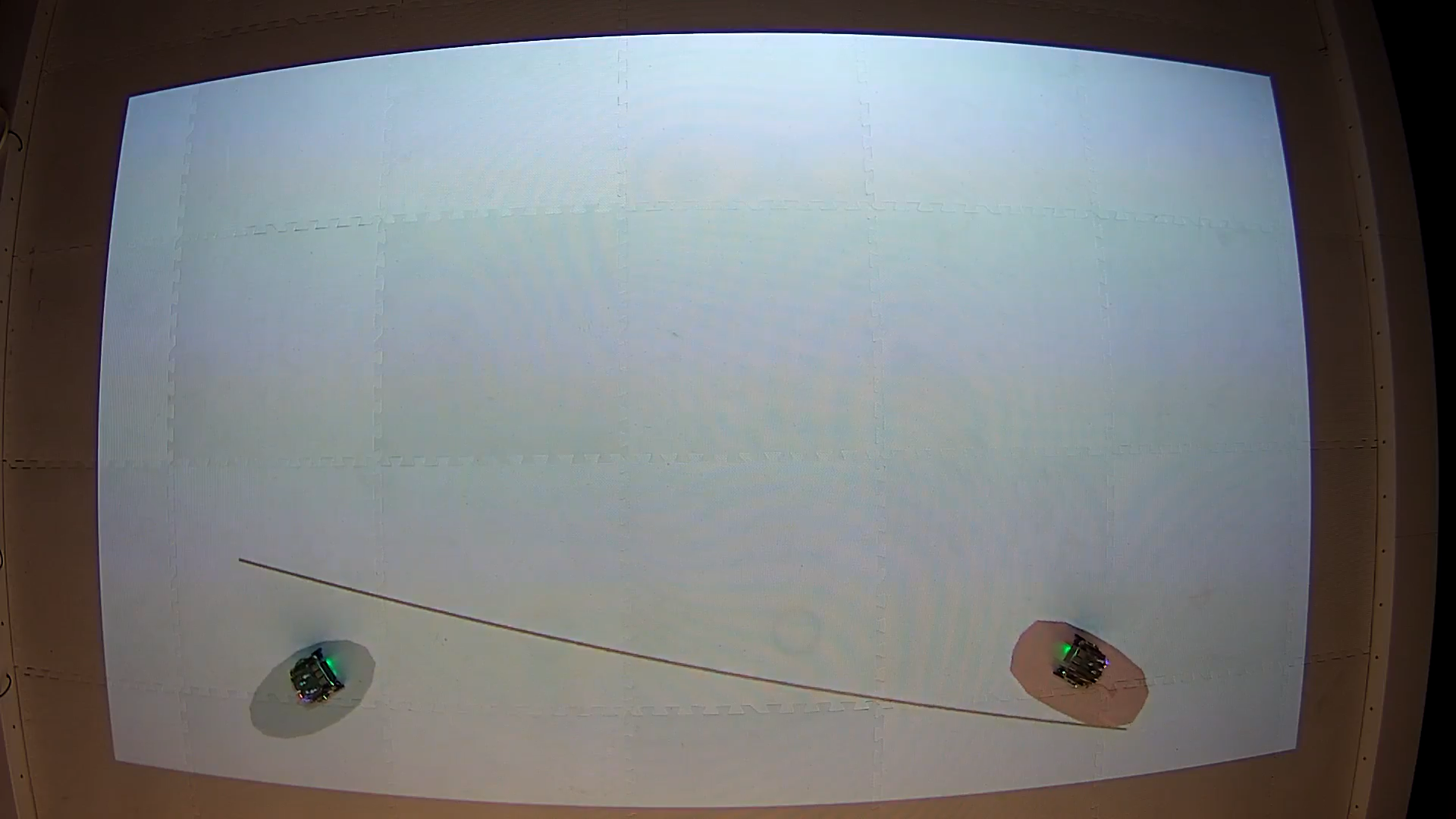}
    \label{fig:experiment:4}}
  \caption{Snapshots of the experiment with two nonholonomic vehicles. During the navigation, a separating hyperplane, colored in black, always separates the two vehicles. The vehicles navigate to their target positions without collision.}
  \label{fig:experiment}
\end{figure*}

Experiments based on the Robotarium platform \cite{pickem2017robotarium} are conducted to demonstrate that our control method can be applied to vehicles with nonholonomic dynamics. For the case $d=2$, the nonholonomic dynamics are described by:

\begin{subequations}
  \label{eqn:nonholonomic dynamics}
  \begin{gather}
    \label{eqn:nonholonomic dynamics:rotational}
    \dot{\bm{R}}_i = \bm{R}_i \begin{bmatrix} 0 & -1 \\ 1 & 0 \end{bmatrix} \omega_i, \\
    \label{eqn:nonholonomic dynamics:translational}
    \dot{\bm{\rho}}_i = \bm{R}_i \cdot \begin{bmatrix} 1 & 0 \\ 0 & L \end{bmatrix} \cdot \begin{bmatrix} v_i \\ \omega_i \end{bmatrix},
  \end{gather}
\end{subequations}
where $\omega_i \in \mathbb{R}$ is the angular velocity in the body frame of vehicle $i$, $v_i \in \mathbb{R}$ is the forward translational velocity of the rear axle in its body frame, and $L \in \mathbb{R}$ is the distance from the center of the rear axle to the geometric center $\bm{\rho}_i$, as illustrated in Fig. \ref{fig:nonholonomic}. The time derivatives of the CBFs \eqref{eqn:collision-free CBF} under the dynamics \eqref{eqn:nonholonomic dynamics} can be derived using the procedures outlined in Appendix \ref{app:proof:lem:time derivatives of CBFs}.

The control objective of the experiment is to drive the geometric center $\bm{\rho}_i$ of vehicle $i$ to its target position $\bm{\rho}_i^d$. This is achieved using the following nominal controller:

\begin{equation}
  \begin{bmatrix} v_i^d \\ \omega_i^d \end{bmatrix} = - k_{\rho} \cdot \begin{bmatrix} 1 & 0 \\ 0 & 1/L \end{bmatrix} \cdot \bm{R}_i^T \cdot (\bm{\rho}_i - \bm{\rho}_i^d), \quad k_{\rho} > 0.
\end{equation}
Fig. \ref{fig:experiment:1} shows the initial configurations of the two differential-driven vehicles, each stabilized at its target position, marked with the corresponding color.
The geometric parameters of the two vehicles are set as $\bm{Q}_1 = [0.4, 0; 0, 0.2]$, $p_1 = 2$, $\bm{Q}_2 = [0.4, 0; 0, 0.2]$, and $p_2 = 3$.
Fig. \ref{fig:evolution of CBFs} depicts the evolution of the CBFs during the experiment. Since the values of the proposed CBFs remain greater than 0 throughout the experiment, no collisions occur between the two vehicles.

\begin{figure}[h!]
  \vspace{-0.7em}
  \centering
  \includegraphics[clip, trim={1.3cm 3.2cm 1.0cm 4.5cm}, width=0.8\hsize]{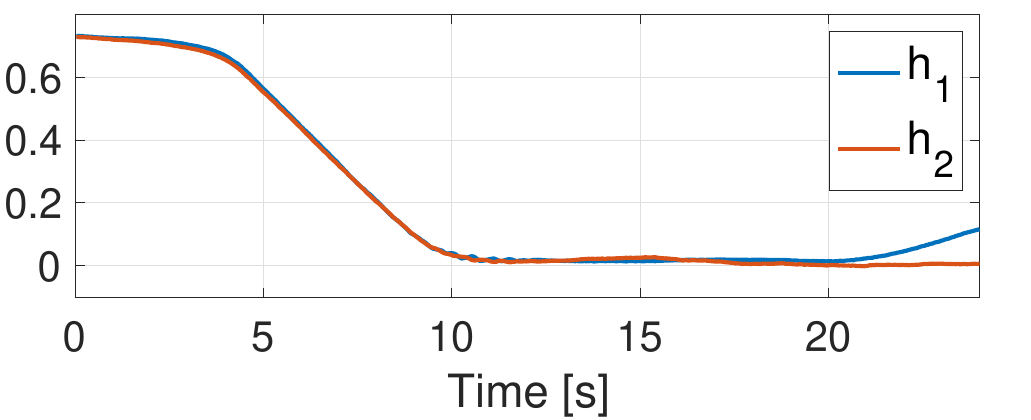}
  \caption{The evolution of CBFs in the experiment.}
  \label{fig:evolution of CBFs}
  \vspace{-0.7em}
\end{figure}

\subsection{Benchmark Comparisons}
To demonstrate the advantages in computational efficiency, the proposed collision avoidance control method is compared with the following state-of-the-art CBF-based collision avoidance methods:

\begin{enumerate}
  \item The duality-based safety-critical control (DB-CBF) \cite{thirugnanam2022duality} approximates the time derivative of CBFs (the Euclidean distance between two objects) using the Lagrange dual problem. This method requires solving a primal problem to obtain the optimal dual variables, in addition to solving the CBF-QP problem, resulting in a double-level optimization process. It is worth noting that the original method is designed for polyhedra and has been customized for general ellipsoids according to the paradigm in \cite{zhang2020optimization}.

  \item The differentiable optimization-based CBFs (DO-CBF) \cite{dai2023safe} derive the time derivative of CBFs (the growth distance \cite{ong1996growth} between two objects) from the KKT conditions. Since the KKT conditions require solving a primal problem first, this method also involves a double-level optimization process. Unlike our method and the method in \cite{thirugnanam2022duality}, this approach does not introduce additional virtual states into the CBF-QP, resulting in a lower-dimensional QP to solve.
\end{enumerate}

All methods are evaluated on a MacBook Pro laptop with an M4 Pro chip and 18 GB of RAM. The optimization problems, including the CBF-QP \eqref{eqn:collision avoidance CBF-QP:origin} and the primal problems of DB-CBF and DO-CBF, are solved using the MOSEK software (version 10.2) with MATLAB interfaces \cite{mosek}.

In each simulation, the scenario is designed as the random position stabilization of $10$ ellipsoids under the dynamics \eqref{eqn:object dynamics} with $d=2$. To achieve collision avoidance between each pair of agents, the DB-CBF method must solve $45$ primal problems, which are conic optimization problems with a dimension of $4$. Similarly, the DO-CBF method must solve $45$ conic optimization problems with a dimension of $3$. All these primal problems are solved sequentially.

In addition to the primal problem, all these methods must solve a CBF-QP to obtain safe control inputs. The DB-CBF method introduces $45 \!\times\! 4$ virtual control inputs in addition to $30$ physical control inputs, resulting in a CBF-QP with a dimension of $210$. Moreover, $45$ extra constraints are introduced into the CBF-QP in addition to the $45$ CBF constraints. The DO-CBF method does not introduce any additional virtual control inputs, and only $45$ CBF constraints are needed. The proposed method introduces $3$ virtual control inputs for each pair of agents and requires $2$ CBF constraints for each pair, resulting in a CBF-QP with a dimension of $165$ and $90$ constraints.

\begin{table}[htbp]
  \tabcolsep=0.1cm 
  \centering
  \caption{Benchmark Comparison}
  \label{tbl:benchmark}
  \begin{tabular}{*{7}{c}}
    \toprule
    \multirow{3}{*}{Method}              & \multicolumn{2}{c}{Primal Problem} & \multicolumn{3}{c}{CBF-QP} & \multirow{3}{*}{Total Time}                                                    \\
    \cmidrule(lr){2-3} \cmidrule(lr){4-6}
                                         & Dim.                               & Avg. Time                  & Dim.                        & Constr.     & Avg. Time                          \\ \midrule
    DB-CBF \cite{thirugnanam2022duality} & $45 \!\times\! 4$                  & 13.8~ms                    & $210$                       & 90          & 14.7 ms         & 32.2~ms          \\
    DO-CBF \cite{dai2023safe}            & $\mathbf{45 \!\times\! 3}$         & \textbf{9.2~ms}            & \textbf{30}                 & \textbf{45} & \textbf{3.3 ms} & 14.1~ms          \\
    Proposed                             & ---                                & ---                        & $165$                       & 90          & 10.9 ms         & \textbf{11.8 ms} \\
    \bottomrule
  \end{tabular}
  \begin{tablenotes}
    \item The bold entities present the best performance results for each column.
  \end{tablenotes}
  \vspace{-0.5em}
\end{table}

Table \ref{tbl:benchmark} reports the dimensions of the optimization problems, the number of constraints, and the average computational times for the different methods. As expected, the total computational time of the DB-CBF method is longer than that of the DO-CBF method and our proposed method, as it involves two high-dimensional optimization problems.
Although the DO-CBF method has a lower-dimensional CBF-QP, it still requires solving an additional primal problem, and derives the time derivative of CBFs from the KKT conditions, which complicates the optimization process, resulting in a longer average time for the CBF-QP compared to the proposed method in this scenarios.

\section{Conclusions and Future Works}
This paper has proposed a collision avoidance control method for general ellipsoids based on the separating hyperplane theorem.
Collision-free CBFs are analytically constructed using collision-free conditions derived from the dual cone.
Thanks to the analytical form of the CBFs, the proposed collision avoidance method does not require solving additional optimization problems beyond the CBF-QP, which expedites the process for achieving safe control.
Simulations and experiments have been conducted with various system dynamics to verify the effectiveness and extendability of the proposed method.

Future work includes extending the method from general ellipsoids to general convex primitives, implementing the proposed control method in a distributed manner, and actively driving the hyperplane to avoid potential deadlocks between objects.

\begin{appendix}
  \subsection{Proof of Lemma \ref{lem:time derivatives of CBFs}}
  \label{app:proof:lem:time derivatives of CBFs}
  The time derivatives of functions $h_i$ and $h_j$ are given by
  \begin{equation}
    \label{eqn:time derivatives of functions}
    \begin{aligned}
      \dot{h}_i & = \frac{\partial h_i}{\partial \lambda_i}^T \dot{\lambda}_i + \frac{\partial h_i}{\partial \bm{\mu}_i}^T \dot{\bm{\mu}}_i, \\
      \dot{h}_j & = \frac{\partial h_j}{\partial \lambda_j}^T \dot{\lambda}_j + \frac{\partial h_j}{\partial \bm{\mu}_j}^T \dot{\bm{\mu}}_j.
    \end{aligned}
  \end{equation}
  Taking the time derivative on both sides of \eqref{eqn:intermediate variables}, the time derivatives of the intermediate variable $\lambda_i$ are obtained as
  \begin{equation*}
    \begin{aligned}
      \dot{\lambda}_i & = \bm{n}_{ij}^T \dot{\bm{\rho}}_i + \bm{\rho}_i^T \dot{\bm{n}}_{ij} + \dot{\gamma}_{ij}                               \\
                      & = \bm{n}_{ij}^T \bm{R}_i \bm{v}_i + \bm{\rho}_i^T (\bm{I}_d - \bm{n}_{ij}\bm{n}_{ij}^T) \bm{\eta}_{ij} - \delta_{ij}.
    \end{aligned}
  \end{equation*}
  Similarly, the time derivative of intermediate variable $\lambda_j$ is
  \begin{equation*}
    \dot{\lambda}_j = -\bm{n}_{ij}^T \bm{R}_i \bm{v}_j - \bm{\rho}_j^T (\bm{I}_d - \bm{n}_{ij}\bm{n}_{ij}^T) \bm{\eta}_{ij} + \delta_{ij}.
  \end{equation*}
  Regarding the time derivative of intermediate variable $\bm{\mu}_i$,
  \begin{equation*}
    \begin{aligned}
      \dot{\bm{\mu}}_{i} & = \bm{Q}_i^T \dot{\bm{R}}_i^T \bm{n}_{ij} + \bm{Q}_i^T \bm{R}_i^T \dot{\bm{n}}_{ij}                                                         \\
                         & = \bm{Q}_i^T \widehat{\bm{\omega}}_i^T \bm{R}_i^T \bm{n}_{ij} + \bm{Q}_i^T \bm{R}_i^T (\bm{I}_d - \bm{n}_{ij}\bm{n}_{ij}^T) \bm{\eta}_{ij}.
    \end{aligned}
  \end{equation*}
  For the case where $d=2$, according to the definition of the operation $\wedge$, we have $\widehat{\bm{\omega}}_i^T \bm{R}_i^T \bm{n}_{ij} = \widehat{1}^T \bm{R}_i^T \bm{n}_{ij} \omega_i$, that is,
  \begin{equation*}
    \dot{\bm{\mu}}_{i} = (\bm{R}_i \widehat{1} \bm{Q}_i)^T \bm{n}_{ij} \omega_i + (\bm{R}_i \bm{Q}_i)^T (\bm{I}_d - \bm{n}_{ij}\bm{n}_{ij}^T) \bm{\eta}_{ij}.
  \end{equation*}
  For the case where $d=3$, according to the definition of the operation $\wedge$, we have $\widehat{\bm{\omega}}_i^T \bm{R}_i^T \bm{n}_{ij} = \widehat{(\bm{R}_i^T \bm{n}_{ij})} \bm{\omega}_i$, that is,
  \begin{equation*}
    \dot{\bm{\mu}}_{i} = \bm{Q}_i^T \widehat{(\bm{R}_i^T \bm{n}_{ij})} \bm{\omega}_i + (\bm{R}_i \bm{Q}_i)^T (\bm{I}_d - \bm{n}_{ij}\bm{n}_{ij}^T) \bm{\eta}_{ij}.
  \end{equation*}
  Similar results for the intermediate variable $\bm{\mu}_j$ can be obtained. For the case where $d=2$,
  \begin{equation*}
    \dot{\bm{\mu}}_{j} = -(\bm{R}_j \widehat{1} \bm{Q}_j)^T \bm{n}_{ij} \omega_j - (\bm{R}_j \bm{Q}_j)^T (\bm{I}_d - \bm{n}_{ij}\bm{n}_{ij}^T) \bm{\eta}_{ij}.
  \end{equation*}
  For the case where $d=3$,
  \begin{equation*}
    \dot{\bm{\mu}}_{j} = -\bm{Q}_j^T \widehat{(\bm{R}_j^T \bm{n}_{ij})} \bm{\omega}_j - (\bm{R}_j \bm{Q}_j)^T (\bm{I}_d - \bm{n}_{ij}\bm{n}_{ij}^T) \bm{\eta}_{ij}.
  \end{equation*}
  Substituting the time derivatives of the intermediate variables into \eqref{eqn:time derivatives of functions}, the coefficient matrices can be obtained.
  The proof is complete.

  \subsection{Proof of Theorem \ref{thm:validity of CBFs}}
  \label{app:proof:thm:validity of CBFs}
  Due to the similarity in proving the conclusions for $h_i$ and $h_j$, we only prove the results for $h_i$. The same results can be obtained straightforwardly for $h_j$.

  According to Lemma \ref{lem:time derivatives of CBFs}, the coefficient vector $(\bm{a}_i, \bm{b}_i, \bm{c}_i, \bm{d}_i)$ is a continuous function with respect to the state of the system if and only if ${\partial h_i}/{\partial \lambda_i}$ and ${\partial h_i}/{\partial \bm{\mu}_i}$ are continuous functions with respect to $(\lambda_i, \bm{\mu}_i)$.
  It can be verified that ${\partial h_i}/{\partial \lambda_i} = 1$ is always continuous, while ${\partial h_i}/{\partial \bm{\mu}_i}$ is continuous unless $\bm{\mu}_i = \bm{0}$.
  By the definition of $\bm{\mu}_i$, the following equality holds:
  \begin{equation}
    \bm{\mu}_i^T \bm{\mu}_i = \bm{n}_{ij}^T (\bm{R}_i \bm{Q}_i) (\bm{R}_i \bm{Q}_i)^T \bm{n}_{ij}.
  \end{equation}
  The matrix $(\bm{R}_i \bm{Q}_i) (\bm{R}_i \bm{Q}_i)^T$ is positive definite since $\bm{R}_i$ and $\bm{Q}_i$ are both invertible matrices.
  Furthermore, the vector $\bm{n}_{ij} \neq \bm{0}$ because the dynamics \eqref{eqn:hyperplane dynamics} preserve its norm.
  Consequently, $\bm{\mu}_i^T \bm{\mu}_i > 0$, and thus $\bm{\mu}_i \neq \bm{0}$.
  This proves that ${\partial h_i}/{\partial \bm{\mu}_i}$ is always continuous under the system dynamics \eqref{eqn:hyperplane dynamics}, which establishes the first point.

  By Lemma \ref{lem:time derivatives of CBFs}, we have $\bm{d}_i = ({\partial h_i}/{\partial \lambda_i})^T = 1$.
  Consequently, the coefficient vectors $(\bm{a}_i, \bm{b}_i, \bm{c}_i, \bm{d}_i)$ cannot become zero, which proves the second point.

  Note that $0 \geq -\alpha(h_i)$ and $0 \geq -\alpha(h_j)$ on $\mathcal{S}_{ij}$.
  As a result, the control input $(\bm{\omega}_i, \bm{v}_i, \bm{\omega}_j, \bm{v}_j, \bm{\eta}_{ij}, \delta_{ij}) = \bm{0}$ is always feasible according to \eqref{eqn:time derivatives of CBFs}.
  The proof is complete.

\end{appendix}

\footnotesize{
  \bibliographystyle{IEEEtranN}
  \bibliography{IEEEabrv, papers.bib}
}

\addtolength{\textheight}{-12cm}
\end{document}